\documentclass{article}

\usepackage{authblk}
\usepackage[round]{natbib}

\usepackage[utf8]{inputenc} 
\usepackage[T1]{fontenc}    
\usepackage{hyperref}       
\usepackage{url}            
\usepackage{booktabs}       
\usepackage{amsfonts}       
\usepackage{nicefrac}       
\usepackage{microtype}      

\usepackage{tikz}

\usepackage{mathtools}
\usepackage{amsfonts}       
\usepackage{bbm}
\usepackage{amsthm}
\newtheorem{thm}{Theorem}
\newtheorem{lem}{Lemma}

\newtheorem{cor}{Corollary}

\newtheorem{aspt}{Assumption}

\theoremstyle{definition}

\DeclareMathOperator*{\argmax}{arg\,max}
\DeclareMathOperator*{\argmin}{arg\,min}
\newcommand\numberthis{\addtocounter{equation}{1}\tag{\theequation}}

\newcommand{\vect}[1]{\boldsymbol{#1}}

\newcommand{\E}{\mathbb{E}}

\usepackage{algorithm}
\usepackage{algcompatible}
\algnewcommand\algorithmicreturn{\textbf{return}}
\algnewcommand\RETURN{\State \algorithmicreturn}%

\usepackage{xcolor}
\newcount\comments  
\newcommand{\genComment}[2]{\ifnum\comments=1{\textcolor{#1}{\textsf{\footnotesize #2}}}\fi}

\title{Decision Making Problems with Funnel Structure: A Multi-Task Learning Approach with Application to Email Marketing Campaigns}

\author[1]{Ziping Xu}
\author[2]{Amirhossein Meisami}
\author[1]{Ambuj Tewari}
\affil[1]{University of Michigan, Department of Statistics}
\affil[2]{Adobe Inc.}
\date{}

\begin{document}

\maketitle

\begin{abstract}
    This paper studies the decision making problem with {\it Funnel Structure}. Funnel structure, a well-known concept in the marketing field, occurs in those systems where the decision maker interacts with the environment in a layered manner receiving far fewer observations from deep layers than shallow ones. For example, in the email marketing campaign application, the layers correspond to Open, Click and Purchase events. Conversions from Click to Purchase happen very infrequently because a purchase cannot be made unless the link in an email is clicked on.

    We formulate this challenging decision making problem as a contextual bandit with funnel structure and develop a multi-task learning algorithm that mitigates the lack of sufficient observations from deeper layers. We analyze both the prediction error and the regret of our algorithms. We verify our theory on prediction errors through a simple simulation. Experiments on both a simulated environment and an environment based on real-world data from a major email marketing company show that our algorithms offer significant improvement over previous methods.

\end{abstract}

\section{Introduction}
We consider decision making problems arising in online recommendation systems or advertising systems \citep{pescher2014consumer,manikrao2005dynamic}. Traditional approaches to these problems only optimize a single reward signal (usually purchase or final conversion), whose positive rate can be extremely low in some real recommendation systems.
This reward sparsity can lead to a slow learning speed and unstable models. Nevertheless, some non-sparse signals are usually available in these applications albeit in a layered manner. These signals can be utilized to boost the performance of the final sparse signal. As a special case, funnel structure generates a sequence of binary signals by layers and the observations are cumulative products of the sequence. An example of email conversion funnel is shown in Figure \ref{fig:email_funnel}.

Funnel structure characterizes a wide range of problems in advertising systems. In the email campaign problem, the learning agent decides the time to send emails to maximize purchases. Apart from the final reward on purchase, we also observe the opening and clicking status of an email.
There are also papers studying the participation funnel in MOOCs \citep{clow2013moocs, borrella2019predict}. Students go through the layers of Awareness, Registration, Activity, Progress and Completion until they drop or complete the course. For both of the funnels, the drop-off fraction at each layer is large. For example, in email campaigns, the conversion rates are typically 10\% for Open, 4\% for Click and 2.5\% for Purchase.

\paragraph{Funnel structure studies in the marketing field.} Conversion funnel has been at the center of the marketing literature for several decades \citep{howard1969jn,barry1987development,mulpuru2011purchase}. This line of work focuses on the attribution of advertising effects and is more interested in analyzing buyers' behavior at each layer. \cite{schwartz2017customer} learns contextual bandit with a Funnel Structure. However, their model directly learns on the final purchase signal and signals on other stages are only used for performance evaluation. Hence, it lacks a comprehensive method that exploits the structural information of a funnel.

\paragraph{Multi-task learning.} Learning on a funnel structure can be seen as a multi-task learning problem, where predicting the conversion on each stage is a single task. Previous literature analyzed how learning among multiple similar tasks can improve sample efficiency \citep{evgeniou2004regularized,ruder2017overview,zhang2017survey}. Most of the multi-task learning analyses require an assumption on the similarity among tasks. \citet{evgeniou2004regularized} assumes that true parameters of different tasks are centered at some unknown point, which we adopted in our method. We also consider a sequential dependency among the tasks, which imposes restrictions on the unknown parameters between adjacent layers.

Analyses of multi-task learning often assume balanced sample sizes for different tasks. Otherwise, their bounds depend on the harmonic mean of sample sizes $\frac{1}{n} \sum_{i=1}^{n} \frac{1}{m_{i}}$, which can be large when sample sizes are imbalanced as in funnel structures. Furthermore, most of the multi-task learning algorithms minimize the average losses over the task set, which is {\em not} our primary goal.

\paragraph{Contextual bandits.}
The decision making problem is modelled as a contextual bandit problem \citep{li2010contextual,li2011unbiased,beygelzimer2011contextual} in our paper. Previous works on contextual bandits mainly focus on a single reward. \citet{drugan2013designing,turugay2018multi} study multi-objective bandits by considering a Pareto regret, which optimizes the vector of rewards for different objectives, while our work focuses on optimizing the final reward by exploiting the whole task set.

\paragraph{Our contributions.}
Our main contributions are summarized below:
\begin{enumerate}
    \item We formally formulated the important {\it Funnel Structure} problem from the marketing field.
    \item We proposed a multi-task learning algorithm for contextual bandits with funnel structure and analyzed the prediction errors and regret under various similarity assumptions.
    \item Our prediction error bound was verified on a simulated environment.
    \item Our algorithm improves previous methods on both simulated and real-data contextual bandit environment.
\end{enumerate}

\section{Formulation}

In this section, we introduce the formulation for funnel structure and discuss how the formulation applies to our email campaign problem. We also introduce the generalized linear model and the assumptions for our theoretical analyses.

\paragraph{Funnel structure.} A funnel, denoted by $F = \{J, \mathcal{X}, (Z_1, \dots, Z_J)\}$, consists of the number of layers $J \in \mathbb{N}$, feature space $\mathcal{X} \in \mathbb{R}^d$ for some $d > 0$ and a sequence of $J$ mappings $(Z_1, \dots, Z_J)$. Each $Z_j$ is a mapping from feature space to $[0, 1]$. On each interaction, a funnel takes an input feature $x \in \mathcal{X}$ and generates a sequence of binary variables $z_1, \dots, z_J$ from Bernoulli distributions with parameters $Z_1(x), \dots, Z_J(x)$, respectively. Then it returns $r_1, \dots r_J$, for $r_j = \prod_{s = 1}^{j} z_s$, to the learning agent.

\paragraph{Email conversion funnel.} We illustrate how the formulation applies to the email conversion funnel. Our email conversion funnel, as shown in Figure \ref{fig:email_funnel}, has 3 layers representing Open, Click and Purchase, respectively. Every email sent to a user randomly generates $r_1, r_2, r_3$ representing whether the email is actually opened, clicked or purchased using the mechanism described above, while $z_1, z_2, z_3$ are the indicators for the three events given previous events happened.

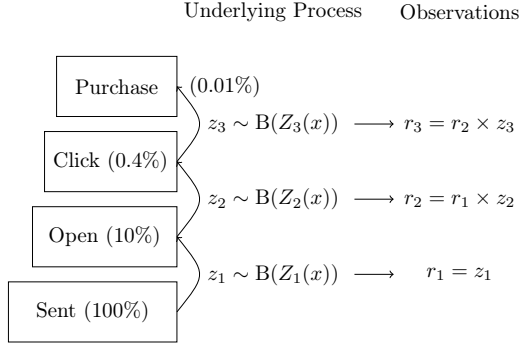
\begin{figure}[t]\scalebox{0.8}
{
\begin{tikzpicture}
\draw (-0.8,0) rectangle node (sent) {Sent (100\%)} ++(2.8,1);
\draw[->]  (2, 0.5) .. controls (2.5, 1.125) .. (2, 1.75);
\draw (-0.4,1.25) rectangle node (open) {Open (10\%)} ++(2.4,1);
\draw[->]  (2, 1.75) .. controls (2.5, 2.375) .. (2, 3);
\draw (-0.2,2.5) rectangle node (click) {Click (0.4\%)} ++(2.2,1);
\draw[->]  (2, 3) .. controls (2.5, 3.625) .. (2, 4.25);
\draw (0,3.75) rectangle node (purchase) {Purchase} ++(2,1);
\draw (2.8,4.25)   node (B) {(0.01\%)};

\begin{scope}[xshift=3.6cm]
\draw (0,5.5)   node (B) {Underlying Process};
\draw (0,1.125) node (B) {$z_1 \sim \text{B}(Z_1(x))$};
\draw (0,2.375) node (B) {$z_2 \sim \text{B}(Z_2(x))$};
\draw (0,3.625) node (B) {$z_3 \sim \text{B}(Z_3(x))$};
\draw[->] (1.35, 1.125) -- (2, 1.125);
\draw[->] (1.35, 2.375) -- (2, 2.375);
\draw[->] (1.35, 3.625) -- (2, 3.625);
\end{scope}

\begin{scope}[xshift=6.7cm]
\draw (0,5.5)   node (B) {Observations};
\draw (0,1.125) node (B) {$r_1 = z_1$};
\draw (0,2.375) node (B) {$r_2 = r_1 \times z_2$};
\draw (0,3.625) node (B) {$r_3 = r_2 \times z_3$};
\end{scope}
\end{tikzpicture}
}
\caption{An illustration of the email conversion funnel. Given any input $x$, the profile information of the user, the funnel generates, $z_1, \dots z_3$, from Bernoulli distributions with parameter $Z_1(x), Z_2(x), Z_3(x)$, representing whether the email would be opened, clicked or purchased given the conversion of the previous layers happened. The observations $r_1, \dots, r_3$ represent whether the email is actually opened, clicked or purchased, respectively.}
\label{fig:email_funnel}
\end{figure}

On the sparsity of the funnel, if an email is never opened, neither click or purchase could happen. Out of all the emails sent to users, 10\% of them were opened, 0.4\% were clicked, and 0.01\% led to a purchase. More generally, when there exists a $r_j = 0$, all the successors $r_i$'s, $i > j$, become 0, which leads to unobservable $z_{j+1}, \dots, z_J$. On average, given a feature $x$, the probability of observing $z_j$ is $P_{j-1}(x)$, which decreases {\bf exponentially} as the layers go deeper.

\paragraph{Contextual bandit with funnel structure.}
Our contextual bandit with funnel structure is denoted by $M = \{\mathcal{A}, \mathcal{X}, P_x, \{F_a\}_{a \in \mathcal{A}}\}$, where $\mathcal{A}$ is the finite action space with $|\mathcal{A}| = A$, $\mathcal{X}$ is the common context space, $P_x$ is the context distribution and each arm $a \in \mathcal{A}$ is assigned a funnel denoted by $F_a = \{J, \mathcal{X}, (Z^a_1, \dots, Z^a_J)\}$. On each round of $t$-th interaction, the environment generates a context $x_t \sim P_x$, the agent takes an action $a_t$ and the funnel $F_{a_t}$ returns the reward vector $(r_{1t}, \dots, r_{Jt})$ taken the input context $x_t$ based on the process described above.

Note that our setting, when $J = 1$, differs from the contextual bandit setting in \cite{chu2011contextual}, where each arm has a unique context but the same mapping from context to reward function.

\paragraph{Assumptions.} Most analyses on multi-task learning assume some similarities among tasks set to allow knowledge transfer. Here we assume a generalized linear model (GLM) and a prior-known hypothesis class over the unknown parameters for all the layers. The hypothesis class characterizes the relatedness across layers.

\begin{aspt}[Generalized linear model]
Assume all $Z_j = \mu(x^T \theta_j^*)$ for some mean function $\mu: \mathbb{R} \mapsto [0, 1]$ and $\theta_j^*$ is the true parameter of layer $j$. A throughout example of this paper is the model for logistic regression, where $\mu(y) = {1}/(1+\exp(-y))$.
\end{aspt}

We also assume that the mean function $\mu$ is Lipschitz continuous and convex.
\begin{aspt} \label{assp:cts_convx}
We assume that $\mu$ is monotonically increasing and $\mu^{\prime}(x) \geq c_{\mu}$ for all $x \in \mathcal{X}$. We also require that function $\mu$ satisfies $|\mu^{\prime}(x)| \leq \kappa$.
\end{aspt}


\begin{aspt}
Assume $\mathcal{X} \subset \{x \in \mathbbm{R}: \|x\| \leq d_x\}$.
\end{aspt}

Generally, we assume that the joint parameter is from a hypothesis class. Two special cases of interest are introduced, upon which we design our practical algorithms.

\begin{aspt}[Similarity assumption]
\label{asp:similarity}
Let $\vect{\theta} = (\theta^{T}_1, \theta^{T}_2, \dots, \theta^{T}_j)^T\in \mathbb{R}^{dJ}$ and $\vect{\theta}^*$ is the joint vector for the true parameters. We assume $\vect{\theta}^* \in \Theta_0 \subset \mathbb{R}^{dJ}$. Throughout the paper, we discuss two special cases:
\begin{enumerate}
    \item \textbf{Sequential dependency:} $\Theta_0 \coloneqq \{ \vect{\theta} \in \mathbb{R}^{dJ}: \|\theta_j- \theta_{j-1}\|_2 \leq q_j, \text{ for } j > 1$ and $\|\theta_1\| \leq q_1$ for some $q_1, \dots q_J \in \mathbb{R}^{+}$.
    \item \textbf{Clustered dependency:} $\Theta_0 \coloneqq \{ \vect{\theta} \in \mathbb{R}^{dJ}: \exists \theta_0 \in \mathbb{R}^d, \|\theta_j- \theta_0\|_2 \leq q_j, \forall j \in [J]\}$ for some $q_1, \dots q_J \in \mathbb{R}^{+}$.
\end{enumerate}
\end{aspt}

For any set $\Theta \subset \mathbb{R}^{dJ}$, we denote the marginal set of task $j$ by $\Theta[j]$ i.e., $\Theta[j] = \{\theta \in \mathbb{R}^d: \exists \vect{\theta} \in \Theta, \vect{\theta}_j = \theta\}$.

We first note that a hypothesis class over the joint parameters is a common assumption in multi-task learning literature \citep{maurer2016benefit,zhang2017survey,pentina2015curriculum}. Also, in another line of work focusing on transfer learning, the theoretical analyses often assume a discrepancy between tasks \citep{wang2019transfer}. We argue in Appendix \ref{app:lem1} that under our GLM assumptions, the discrepancy assumption is almost the same as ours.


\section{Supervised learning}

Before discussing the contextual bandits with funnel structure, we first consider the supervised learning scenario for a single funnel and  seek a bound for the prediction error of each layer:
$$
    \text{PE}_j = |\mu(x^T \theta_j^*) - \mu(x^T \hat \theta_j)|,
$$
for some estimates $\hat\theta_1, \dots, \hat \theta_J \in \mathbb{R}^d$.

For a single funnel, our algorithm learns on the dataset $\{x_i, r_{1i}, \dots, r_{Ji}\}_{i = 1}^n$ of size $n$. Let $n_j = \sum_{i = 1}^n \mathbbm{1}(r_{j-1, i} = 1)$ be the number available observations for layer $j$
and $j_1, \dots j_{n_j}$ be the indices of these $n_j$ samples, i.e. $r_{j-1, j_i} > 0$ for all $i \in [n_j]$. Let $z_{j, j_i} = r_{j, j_i} / r_{j-1, j_i}$. 
We denote the square loss function of layers $j$ by
\begin{equation}
\label{equ:loss}
    l_j(\theta) \coloneqq  \sum_{i = 1}^{n_j} (z_{j, j_i} - \mu(x_{j_i}^T\theta))^2.
\end{equation}

\subsection{Implications from a single-layered case} We first investigate a single-layered case and see how prior knowledge helps improve the upper bound on prediction error. Lemma \ref{lem:pred_err} bounds the prediction error using either prior knowledge or collected samples. The proof of Lemma \ref{lem:pred_err} is provided in Appendix \ref{app:lem1}.

\begin{lem}
\label{lem:pred_err}
Using the model defined above, assume $J = 1$ and the true parameter $\theta^* \in \Theta_0$. Let the dataset be $\{x_i, z_i\}_{i = 1}^n$ and let $\Tilde{\theta}$ be the solution that maximizes the L$_2$ function in (\ref{equ:loss}) and $\hat \theta$ be its projection onto $\Theta_0$. Let $q \coloneqq \sup_{\theta_1, \theta_2 \in \Theta_0} \|\theta_1 - \theta_2\|_2$. Then with a probability at least $1-\delta$, we have
\begin{align*}
    &|\mu(x^T\hat \theta) - \mu(x^T\theta^*)| \\
    &\leq \kappa  \min \{
    \sup_{\theta_1, \theta_2 \in \Theta_0} |x^T(\theta_1 - \theta_2)|,
    \|x\|_{M_n^{-1}} \frac{4c_{\delta}}{c_{\mu}}\sqrt{\frac{d}{n \vee 1}}\}\\
    &\leq \kappa  \min \{
    d_x q,
    \|x\|_{M_n^{-1}} \frac{c_{\delta}}{c_{\mu}}\sqrt{\frac{d}{n}}\},
    \numberthis \label{equ:lemma1} \\
    &\text{ furthermore, we have a confidence set on $\theta^*$}\\
    & \theta^* \in \{\theta: \|\hat \theta - \theta\|_{M_n} \leq
    \frac{c_{\delta}}{c_{\mu}}\sqrt{\frac{d}{n}}\}, \numberthis \label{equ:lemma1_2}
\end{align*}
where $M_{n} = \frac{1}{n}\sum_{i = 1}^{n}x_i x_i^{T}$, $c_{\delta} = 80d_x\sqrt{2\ln(8/\delta)}$.
\end{lem}

Equation (\ref{equ:lemma1}) bounds the prediction error with a minimum of two terms. The first term in (\ref{equ:lemma1}) is directly derived from prior knowledge. The second term is a parametric bound without any regularization \citep{srebro2010optimistic}. In fact, (\ref{equ:lemma1}) is a tight upper bound on prediction error as shown in Appendix \ref{app:tightness}.

Lemma \ref{lem:pred_err} implies a {\bf "transfer or learn"} scenario: when the sample size for the new task is not large enough, i.e. $n = o(1/q^2)$, it is more beneficial to directly apply the prior knowledge. Otherwise, one can drop the prior knowledge and use parametric bound.

\subsection{Multi-task learning algorithm}

Our algorithm, inspired by the "transfer or learn" idea, consists of two steps: 1) optimize the loss function for each layer within its marginal set and calculate the confidence set defined in (\ref{equ:lemma1_2}); 2) take the intersection between the confidence set and $\Theta_0$, which generates $\Theta_1$. Then project the unconstrained solution onto the new set $\Theta_1$. The details are shown in Algorithm \ref{algo:mt_funnel}, where $\text{Proj}_\Theta(\theta)$ denotes the projection of $\theta$ onto the set $\Theta$.

\begin{algorithm}[t]
  \caption{Regularized MTL for funnel structure}\label{algo:mt_funnel}
  \begin{algorithmic}
    \STATE Input: number of layers $J$, hypothesis set $\Theta_0$, dataset $\left\{x_{i}, r_{1 i}, \ldots, r_{J i}\right\}_{i=1}^{n}$ generated from the funnel, accuracy $\delta > 0$.
    \STATE {\it \# Calculate confidence set}.
    \FOR{$j = 1$ to $J$}
        \STATE Solve
        $
            \bar\theta_j = \text{Proj}_{ \Theta_0[j]}(\argmin l_j(\theta)).
        $
        \STATE Calculate $\hat\Theta_j$ defined in Equation (\ref{equ:lemma1_2}) by
        $$
            \hat\Theta_j = \{\theta: \|\theta - \bar\theta_j\|_{M_{j, n_j}} \leq \frac{c_{\delta}}{c_{\mu}} \sqrt{\frac{d}{n_j}}\},
        $$
        \STATE for $M_{j, n_j} = \sum_{i = 1}^{n_j} x_{j, j_i}x_{j, j_i}^T$.
    \ENDFOR
    \STATE {\it \# Re-estimate parameters}.
    \STATE Calculate joint set $\hat\Theta = \{\vect{\theta}: \vect{\theta}_j \in \hat \Theta_j \text{ for all } j \in [J]\}$.
    \STATE Set $\Theta_1 \leftarrow \Theta_0 \cap \hat\Theta$.
    \FOR{$j = 1$ to $J$}
        \STATE Solve
        $
            \hat \theta_j = \text{Proj}_{ \Theta_1[j]}( \argmin l_j(\theta)).
        $
    \ENDFOR
    \RETURN{\ $\hat\theta_1, \dots, \hat \theta_J$.}
  \end{algorithmic}
\end{algorithm}

\subsection{Upper bound on prediction error}

Directly applying Lemma \ref{lem:pred_err} within the set $\Theta_1$ gives us a bound on prediction error depending on the marginal set $\Theta_1[j]$ and $n_j$ as shown in Corollary \ref{cor:1}.

\begin{cor}
\label{cor:1}
Let $\hat \theta_1, \dots, \hat \theta_J$ be the estimates from Algorithm \ref{algo:mt_funnel} and $\Theta_1$ be the set defined in Algorithm \ref{algo:mt_funnel}. With a probability at least $1-\delta$, for all $j \in [J]$, we have
\begin{align*}
    \text{PE}_j
    \leq \kappa  \min \{
    & \sup_{\theta_1, \theta_2 \in \Theta_1[j]} |x^T(\theta_1 - \theta_2)|, \\
    & \|x\|_{M_n^{-1}} \frac{c_{\delta}}{c_{\mu}}\sqrt{\frac{d}{n_j}}\}. \numberthis \label{equ:gen_pred_err}
\end{align*}
\end{cor}

However, it is more interesting to discuss the actual form of $\Theta_1 = \Theta_0 \cap \hat \Theta$ and its interactions with $n_j$ under some special assumptions on $\Theta_0$. As mentioned in Assumption \ref{asp:similarity}, we consider two cases: sequential dependency and clustered dependency.

Recall that for the sequential dependency, we assume $\vect{\theta}^* \in \Theta_0 \coloneqq \{ \vect{\theta} \in \mathbb{R}^{dJ}: \|\theta_j- \theta_{j-1}\|_2 \leq q_j, \text{ for } j > 1 \text{ and } \|\theta_1\|_2 \leq q_1\}$ for some $q_1, \dots q_J > 0$. Before presenting our results, we need an extra assumption on the distribution of covariates.

\begin{aspt}\label{asp:covairate_dist}
Assume the minimum eigenvalue of $M_{j, n_j}$ is lower bounded by a constant $\lambda > 0$ for all $j \in [J]$.
\end{aspt}

Assumption \ref{asp:covairate_dist} guarantees that the distribution of the covariate covers all dimensions.

\begin{thm}[Prediction error under sequential dependency]
\label{thm:seq_thm}
For any funnel with a sequential dependency of parameters $q_1, \dots, q_J$, let $\hat \theta_1, \dots, \hat \theta_J$ be the estimates from Algorithm \ref{algo:mt_funnel}. If $n_{j+1} \leq n_j / 4$, $q_1 \geq, \dots, \geq q_J $ and Assumption 5 is satisfied, then with a probability at least $1-\delta$, for any $j_0 \in [J]$, we have
\begin{align*}
    &\text{PE}_j \leq
    \left\{
        \begin{array}{ll}
             & \kappa \|x\|_2 \frac{ c_{\delta/J}}{c_{\mu}\lambda} \sqrt{\frac{d}{n_j}}, \text{ if } j < j_0,  \\
             & \kappa\|x\|_2(\frac{ c_{\delta/J}}{c_{\mu}\lambda} \sqrt{\frac{d}{n_{j_0}}} + \sum_{i = j_0 + 1}^{j} q_i), \text{ if } j \geq j_0,
        \end{array}
    \right.
\end{align*}
where we let $n_0 = \infty$.
The smallest bound of all choices of $j_0$ is achieved when $j_0$ is the smallest $j \in [J]$, such that
\begin{equation}
    \frac{ c_{\delta}\sqrt{d}}{c_{\mu} \lambda} (\frac{1}{\sqrt{n_{j}}} - \frac{1}{\sqrt{n_{j-1}}}) \geq q_{j}, \label{equ:condition}
\end{equation}
if none of $j$'s in $[J]$ satisfies (\ref{equ:condition}), $j_0 = J+1$.
\end{thm}

Theorem \ref{thm:seq_thm} shows that for some funnel under sequential dependency assumption, there exists a threshold layer $j_0$, before which the bounds without multi-task learning are tighter. After $j_0$, we use the bounds depending on prior knowledge. For small $n_j$'s, $j_0 = 1$ and for sufficient large $n_j$'s, $j_0 = J+1$. Figure \ref{fig:seq_dep} shows an example of how the threshold $j_0$ changes when total number $n$ increases in a 5-layered funnel.

\begin{figure}[ht]
    \centering
    \includegraphics[width = 0.85\textwidth]{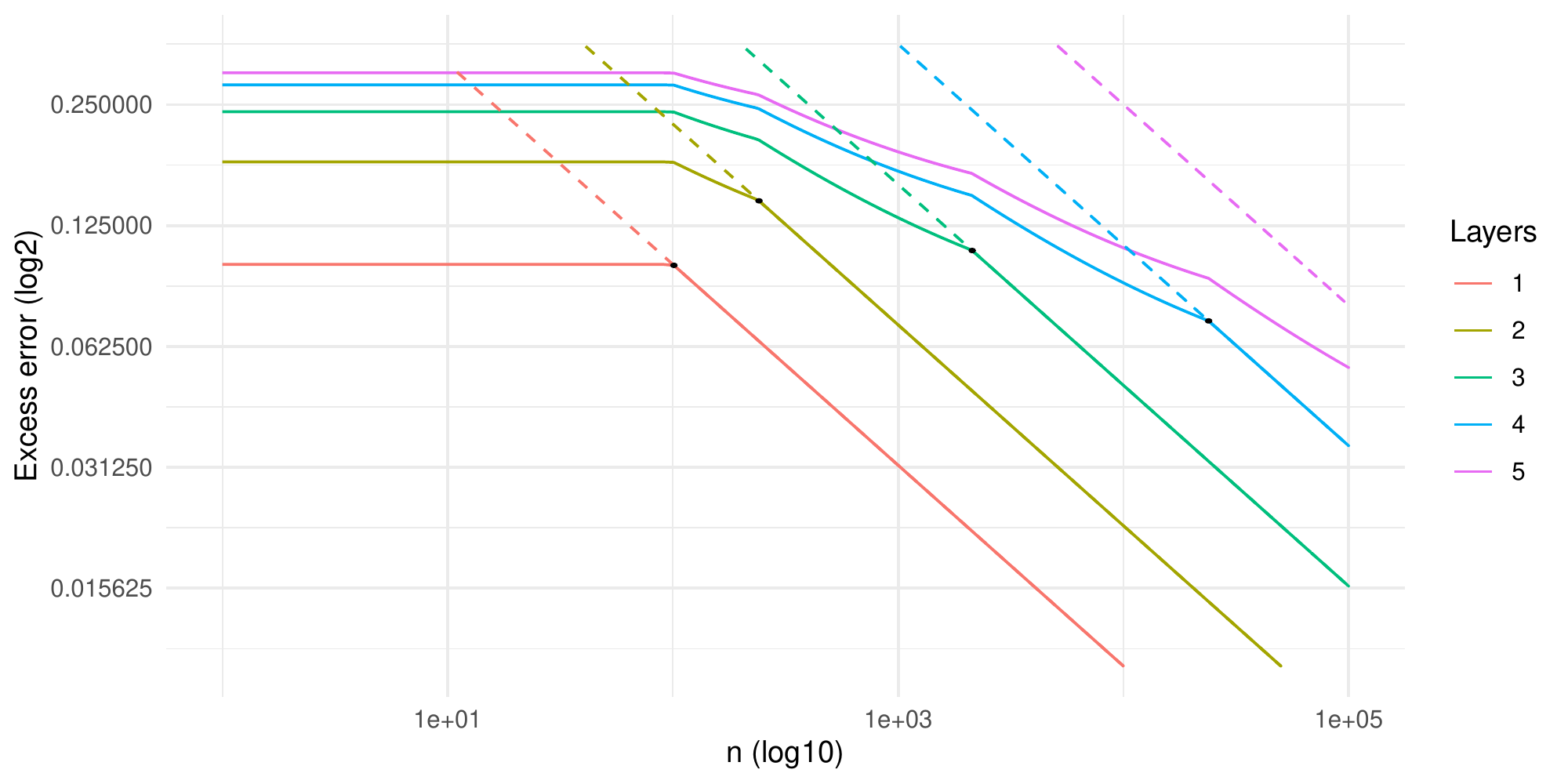}
    \caption{An example of how the prediction error bound in Theorem \ref{thm:seq_thm} changes when the number of observations increases in a 5-layered funnel. We set $\|x\|_2 { c_{\delta}\sqrt{d}}/(c_{\mu}\lambda) = 1$, $q_j = (12 - 2j) / 100$ and $n_j = 0.2^{j-1} n$. Solid lines mark the prediction error bound defined in Theorem \ref{thm:seq_thm} and dashed lines mark the prediction error without multi-task learning (second term in (\ref{equ:lemma1_2})). Black points marked the change of $j_0$.}
    \label{fig:seq_dep}
\end{figure}

For the clustered dependency, the prediction error bound can be characterized by Theorem \ref{thm:clst_thm}.
\begin{thm}[Prediction error under clustered dependency]
\label{thm:clst_thm}
For any funnel with a clustered dependency of parameters $q_1, \dots, q_J$, let $\hat \theta_1, \dots, \hat \theta_J$ be the estimates from Algorithm \ref{algo:mt_funnel}. With a probability at least $1-\delta$,
\begin{align*}
    &\text{PE}_j \leq  \kappa\|x\|_{2}\min\{\frac{c_{\delta / J}}{c_{\mu} \lambda} \sqrt{\frac{d}{n_{j_{0}}}}+q_j, \frac{c_{\delta / J}}{c_{\mu} \lambda} \sqrt{\frac{d}{n_{j}}}\},
\end{align*}
where
$
    j_0 = \argmin_{j \in [J]} \frac{c_{\delta}}{c_{\mu} \lambda} \sqrt{\frac{d}{n_{j}}} + q_j.
$
\end{thm}
Under the clustered dependency, there is a single layer that gives the tightest confidence set on the unknown center, which is used by all the other layers.

\section{Regret Analysis for Contextual Bandit}
In this section, we bound regrets for contextual bandits with funnel structure and discuss the benefits of multi-task learning.

\paragraph{Extra notations.} For simpler demonstration, we define $P_j: \mathcal{X} \mapsto [0, 1]$, such that $P_{j}(x) = \prod_{i = 1}^{j}Z_i(x)$. 
For any $\vect{\theta} = (\theta^{T}_1, \theta^{T}_2, \dots, \theta^{T}_j)^T\in \mathbb{R}^{dJ}$, let $P_j(x, \vect{\theta}) = \prod_{i=1}^j \mu(x^T\theta_j)$. To account for multiple funnels, we let $n^t_{a, j}$ be the number of observations for the $j$-th layer of funnel $F_a$ up to step $t$ . Let $\theta_{a, j}^*$ be the true parameters and $\vect{\theta}_{a}^*$ be the joint vector. Further we let $\lambda_{a, j}^t$ be the sample minimum eigenvalue covariance matrix of layer $j$ of funnel $F_a$, $\lambda_{a, j}$ be the minimum eigenvalue of its expectation and $\bar \lambda$ be a lower bound over all $a$ and $j$. 

Regret of contextual bandits with funnel structure is defined as
$$
    \sum_{t=1}^{T}\left[P_J(x_t, \vect{\theta}_{a_t^*}^*)- P_J(x_t, \vect{\theta}_{a_t}^*)\right],
$$
where $a_t^*$ is the optimal action for input $x_t$, $a_t$ is the the action chosen by the agent at step $t$ and $T$ is the total steps.

\subsection{Optimistic algorithm}
We propose a variation of the famous UCB (upper confidence bound) algorithm that adds bonuses based on the uncertainty of the whole funnel. The prediction error of funnel $F_a$ of a given input $x$ is
$
    |P(x, \hat{\vect{\theta}}_{a}^t) - P(x, {\vect{\theta}}_{a}^*)|
$ for $\hat{\vect{\theta}}_{a}^t$ which is the joint vector of estimates $\hat \theta_{a, j}^t$ at the step $t$.

From Lemma \ref{lem:pred_err}, we define
$$
    \Delta \mu^t_{a, j} = \kappa \|x_t\|_2 \min \left\{\underset{\theta_{1}, \theta_{2} \in \Theta_{a, 1}^t}{\sup }\|\theta_{1}-\theta_{2}\|_2, \frac{c_{\delta / 3AJT}}{c_{\mu}\lambda_{a, j}^t} \sqrt{\frac{d}{n^t_{a, j}\vee 1}}\right\},
$$
where $\Theta_{a, 1}^t$ is the intersection set from Algorithm \ref{algo:mt_funnel} for funnel $a$ at step $t$. Using simple Taylor's expansion, we have Lemma \ref{lem:expansion}.
\begin{lem}
\label{lem:expansion}
Using the estimates from Algorithm \ref{algo:mt_funnel}, we have, with the same probability in (\ref{equ:gen_pred_err}),
\begin{align*}
    &\quad |P(x, \hat{\vect{\theta}}_{a}^t) - P(x, {\vect{\theta}}_{a}^*)| \leq  \\
    &\sum_{j} \frac{P_J(x, \hat{\vect{\theta}}^t_{a})}{\mu(x^T \hat \theta^t_{a, j})} \Delta \mu^t_{a, j} + \sum_{i \neq j} \Delta {\mu}^t_{a, j} \Delta {\mu}^t_{a, i} \eqqcolon \Delta \mu_{a}^t. \numberthis \label{equ:whole_pred_error}
\end{align*}
\end{lem}

Define $P_J^{+}(x, \hat{\vect{\theta}}^t_a) = P_J(x, \hat{\vect{\theta}}^t_a) + \Delta \mu_{a}^t$. We are able to derive an optimistic algorithm shown in Algorithm \ref{algo:mt_cb}. Now we use the prediction error on the whole funnel to analyze the regret.

\begin{algorithm}[t]
  \caption{Contextual Bandit with a Funnel Structure}\label{algo:mt_cb}
  \begin{algorithmic}
    \STATE $t \rightarrow 1$, total number of steps $T$, memory $\mathcal{H}_{a} = \{\}$ for all $a \in [A]$. Initialize $\theta_{a, \star}$ with zero vectors.
    \FOR{$t = 1$ to $T$}
        \STATE Receive context $x_t$.
        \STATE Compute $P_{J}^{+}(x_t, \hat{\vect{\theta}}_{a})$ based on (\ref{equ:whole_pred_error}) for all $a \in [A]$.
        \STATE Choose $a_t = \argmax_{a \in \mathcal{A}} \hat P_{J}^{+}(x_t, \hat\theta_{a, j})$.
        \STATE Receive $r_{t,1}, \dots, r_{t,J}$ from funnel $F_{a_t}$.
        \STATE Set $\mathcal{H}_{a_t} \rightarrow \mathcal{H}_{a_t} \cup \{(x_t, (r_{t,1}, \dots, r_{t,J}))\}$.
        \STATE Update $\hat \theta_{a_t, \star}$ by algorithm \ref{algo:mt_funnel} with dataset $\mathcal{H}_{a_t}$.
    \ENDFOR
  \end{algorithmic}
\end{algorithm}

\subsection{Regret analysis}
 Theorem \ref{thm:regret} bounds regrets for Algorithm \ref{algo:mt_cb}. The regret in Theorem \ref{thm:regret} can be bounded by three terms in (\ref{equ:regret}). The first term in (\ref{equ:regret}) represents the normal regret without any multi-task learning with an order $\mathcal{O}(\sum_{a, j} \sqrt{n_{a, j}^{T}})$, which reduces to the standard $\sqrt{AT}$ when $J = 1$. Note that the impact of one layer on regret is bounded with the square root of its number of observations. The second term is a constant term that does not depend on $T$. The third term represents the benefits of multi-task learning. Full version of Theorem \ref{thm:regret} is given and proved in Appendix \ref{app:thm_regret}.

\begin{thm}
\label{thm:regret}
Using Algorithm \ref{algo:mt_cb}, under the Assumptions 1-4, with high probability, the total regret
\begin{align*}
    &\quad \sum_{t=1}^{T}\left[P(x_t, \vect{\theta}_{a_t^*}^*)- P(x_t, \vect{\theta}_{a_t}^*)\right] \\
    &= \mathcal{O} (c_{0} \sum_{a, j} \sqrt{n_{a, j}^{T}} +  \sum_{a, j} \frac{c_{0}^{2} J d_{x}^{4}}{\bar{p}_{a, j}^{2}})  - \sum_{a, j} \Delta_{a, j}
    \numberthis \label{equ:regret}
\end{align*}
where $\mathcal{O}$ ignores all the constant terms and logarithmic terms for better demonstrations, $c_0 = ({\kappa d_x c_{\delta/3AJT} \sqrt{d}})/({c_{\mu} \bar{\lambda}})$, $\bar p_{a, j} \coloneqq \E_x P_{j-1}(x,  \vect{\theta}_{a}^{*})$ and
$$
    \Delta_{a, j} = \sum_{t = 1; a_t = a}^T P_j(x_t^T \hat\theta_{a_t}^t) \left[c_0 \frac{1}{\sqrt{n^t_{a, j}\vee 1}} - \Delta \mu_{a, j}^t \right].
$$
represents the benefits of transfer learning.
\label{thm:2}
\end{thm}

We discuss the actual form of benefits under sequential dependency. If the hypothesis class is truly sequential dependency with parameters $q_{a, 1}, \dots, q_{a, J}$ for each funnel $F_a$, we have
$$
    \Delta_{a, j} = \mathcal{O}(\sum_{j' \leq j} 1/q_{a, j'}),
$$
for sufficient large $T$. Generally, for sufficient large total steps, the benefits scale with $\sum_{a, j} (J-j+1)/q_{a, j}$.  


\section{Experiments}

In this section, we first present a simulation on the power of multi-task learning using Algorithm \ref{algo:mt_funnel}. Then we propose a more practical contextual bandit algorithm, which is tested on a simulated environment and our real-data email campaign environment.

\subsection{Simulation on supervised learning}

We use a simulation to verify the bound in Theorem \ref{thm:seq_thm} and the curves in Figure \ref{fig:seq_dep}.
We consider a 5-layered funnel with sequential dependency. The link function is $\mu(x) = 1/(1+\exp(-x))$. We set $d = 5$ and $\theta_{j+1} = \theta_{j} + u_j q_j$, where $u_j \in \mathbbm{R}^d$ is a unit vector with a random direction and $q_j = 1.2 - 0.2j$. In the simulation, we apply Algorithm \ref{algo:mt_funnel} under the sequential dependency and calculate the estimation error of the estimates without parameter transfer ($\bar\theta_j$) and the final estimates $\hat \theta_j$. The results are shown in Figure \ref{fig:sim_excess_error}. We observe a similar pattern as in Figure \ref{fig:seq_dep} and the errors of deeper layers are controlled well despite of their small sample sizes.

\begin{figure}
    \centering
    \includegraphics[width = 0.85\textwidth]{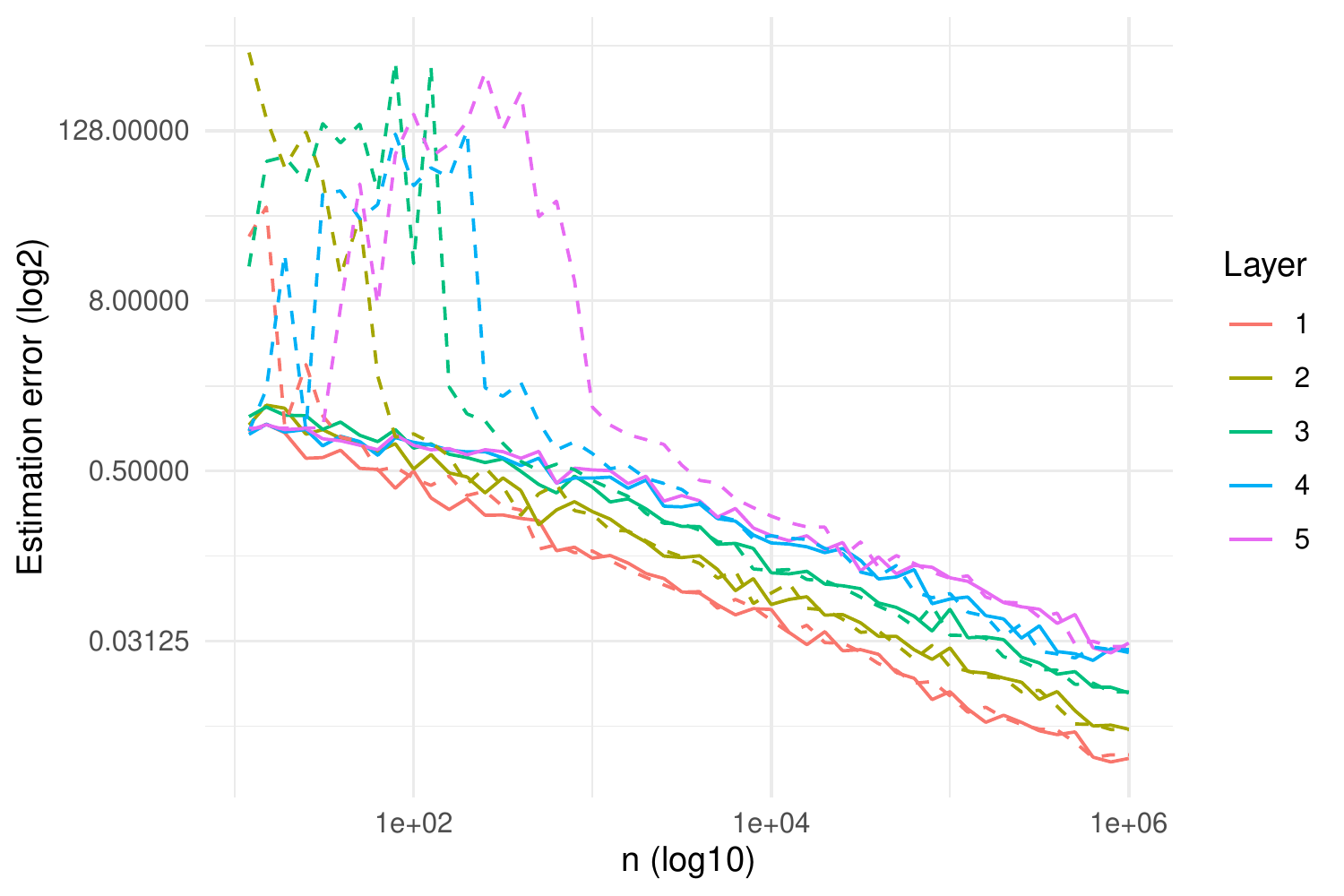}
    \caption{Estimation errors ($L_2$ distance to $\theta_j^*$) of $\bar\theta_j$ and $\hat\theta_j$ under different number of interactions with the funnel. Colors represents the layers. Solid (Dashed) lines represents the estimation errors of $\hat \theta_j$ ($\bar \theta_j$). Each point in the plot is an average over 10 independent runs.}
    \label{fig:sim_excess_error}
\end{figure}

\subsection{Simulations on contextual bandits}

\paragraph{Practical algorithm.} Calculating the intersection between two sets is not easy when $\Theta_0$ has a complex form. Also, we may not have access to $\Theta_0$ in real data analysis. We, therefore, develop practical algorithms especially for the sequential dependency and clustered dependency. Both of the algorithms reduced to optimizing parameters under a L$_2$ regularization. An equivalent form is to optimize under L$_2$ penalty. Our practical contextual bandit algorithm optimize loss function using an L$_2$ penalty controlled by tuned hyper-parameters (Algorithm \ref{algo:prac_mt_cb} in Appendix \ref{app:prac_mt_cb}). Since exploration is not the primary interest of this work, we also adopt a $\epsilon$-greedy exploration for the simplicity of implementation.

\paragraph{Compared algorithms.}
Apart from the naive algorithm that directly learns on the signals $r_J$, some methods learn on the averaged rewards across the funnel. This approach is commonly used in Reinforcement Learning with auxiliary rewards \citep{jaderberg2016reinforcement,lin2019adaptive}, where the true reward is sparse and there are some non-sparse auxiliary rewards that can accelerate learning. We call this method {\it Mix} in the following experiments.

Inspired by the idea of {\it Mix} and that of curriculum learning \cite{bengio2009curriculum}, we also test the method that learns on the signals for each layer sequentially.

To sum up, we compare the following five strategies:
\begin{enumerate}
    \item {\it Target}: we train a single model that only predicts the reward from the last stage, i.e. $r_J$.
    \item {\it Mix}: we train a single model that only predicts the average rewards from all the stages, i.e. $\frac{1}{J} \sum_{j = 1}^J r_j)$
    \item {\it Sequential}: for a total steps $T$, we train a single model on rewards $r_1, \dots, r_{J - 1}$ sequentially for equal number of steps $\alpha T / (J - 1)$ for some constant $\alpha \in [0, 1]$ and train the model on the final reward $r_J$ for the rest of steps.
    \item {\it Multi-layer clustered}: Algorithm \ref{algo:prac_mt_cb} under clustered dependency.
    \item {\it Multi-layer sequential}: Algorithm \ref{algo:prac_mt_cb} under sequential dependency.
\end{enumerate}

\paragraph{Simulated Environments.}
We first tested the performance of multi-task learning algorithms on the contextual bandit setting. In our contextual bandit setting, number of action is set to be $A = 50$, for each action, we independently generate a funnel with $J = 8$ stages. The link function is from the logistic regression,
where we sample the unknown parameter $\theta_{a, j}$ sequentially. In this case, $\theta_{a, 1}$ is sampled from $N(0, \sigma^2)$ and $\theta_{a, j}$ is sampled from $N(\theta_{a, j-1}, \sigma^2 / j)$ for  for $j > 1$. This gives us a funnel with decreasing uncertainty. Context $x$ is sampled from a Gaussian distribution $N(0, \sigma_x^2)$.

We set the parameters for the environment to be $A = 50; J = 8; \sigma = 1; \sigma_x = 0.08; d = 45; T = 3000$. 

\paragraph{Model setup.} For {\it Target}, {\it Mix} and {\it Sequential}, we use a Neural Network model with one hidden layer, $d$-dimensional input and $A$-dimensional output. Each dimension on the output vector represents the predicted conversion probability for an action. The number of units for the hidden layer is searched in $\{8, 16, 32, 64\}$. {\it Sequential} has the hyper-parameter $a$, which is searched in $\{0.1, 0.2, 0.4, 0.6\}$.

For {\it Multi-layer clustered} and {\it Multi-layer sequential}, each stage is modeled with the same one-layer Neural Networks defined above. The number of units for the hidden layer is searched in $\{1, 4, 8, 16\}$. The penalty parameter $\lambda$ is searched in $\{0.001, 0.005, 0.01, 0.05\}$. An $\epsilon$-greedy exploration is applied.


\begin{figure}[ht]
    \centering
    \includegraphics[width = 0.85\textwidth]{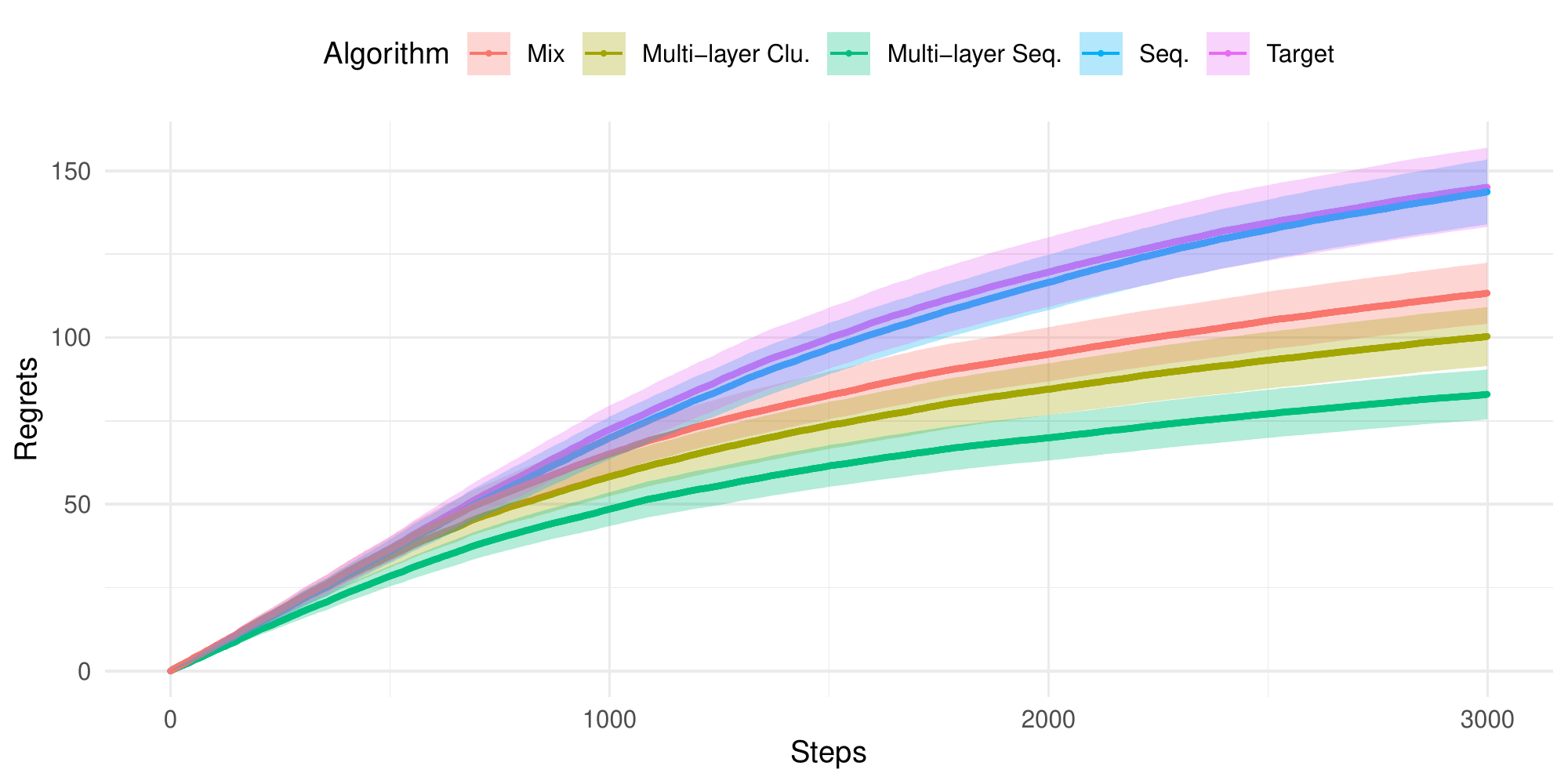}
    \caption{Cumulative regrets over 3000 steps using the best hyper-parameters for each of the five algorithms. The confidence interval is calculated from  independent runs.}
    \label{fig:plot2}
\end{figure}

As shown in Figure \ref{fig:plot2}, our practical algorithms beat all the other algorithms in terms of cumulative regrets. Algorithm \ref{algo:prac_mt_cb} under sequential has lower regrets compared to that under clustered dependency.

\subsection{Email campaign environment}

We further test our algorithms on the Email Campaign problem, which aims at the act of sending a commercial message, typically to a group of people, using email.

\paragraph{Dataset.} We randomly selected 5609706 users, who were active up to the data collection date and then tracked all the interactions of those users in the following 51 days, which adds up to 39488647 emails. Each email has a five-dimensional context, consisting of:     {\it NumSent}, the number of emails sent to the user since 2019-12-01;
    {\it NumOpen}, the number of emails opened by the user since 2019-12-01;
    {\it NumClick}, the number of emails whose links were clicked by the user since 2019-12-01;
    {\it BussinessGroup}, categorical variable indicating the business group of the user; and
    {\it Recency}, number of hours since last email was sent to the user, which is categorized into '0-12', '13-24', '25-36', '37-48', '49+'.

The agent takes actions of the time to send an email. The action space is divided into six blocks: 00:00-04:00, 04:00-08:00, 08:00-12:00, 12:00-16:00, 16:00-20:00, 20:00-24:00.

Three rewards are available for each email, indicating whether the email is opened, whether the email is clicked and whether the email leads to a purchase respectively. Note that we say an email leads a purchase if the user purchases in the following 30 days.

The email campaign problem defines a funnel with three layers. On average, the rates for opening, clicking and purchasing are about 10\%, 0.4\% and 0.01\%, respectively in the dataset. As we can see, the signals for learning purchasing behaviour are sparse. We will show in our experiments that how multi-task learning can improve the sample efficiency.

\paragraph{Data-based environment.} We first build a data-based contextual bandit environment using population distribution. At each step, the environment randomly samples a context from the dataset and the agent takes an action from the six blocks. The environment then samples a reward vector from the set of rewards with the same action and context.

\begin{table}[ht]
\begin{tabular}{l|llllll}
                 & Purchase  & Click & Open  \\
                 & ($10^{-2}$) & &   \\
\hline
Target           & 4.09   & 3.88  & 24.6  \\
Mix              & 4.23   & \textcolor{red}{6.63}  &  \textcolor{red}{47.0}  \\
Sequential       & 2.1    & 0.981 & 0.495 \\
Multi-layer clu.       & \textcolor{red}{6.34}   & 0.753 & -18.1 \\
Multi-layer seq. & 1.06   & 2.04  & 3.03
\end{tabular}
\caption{Average increases in the number of Purchase, Click or Open over 10000 steps compared to the Random policy using 20 independent runs. The standard deviations are all less than $10^{-3}$ for Purchase and $10^{-1}$ for Click and Open.}
\label{tab:email_results}
\end{table}

\paragraph{Results.}
We searched the same set of hyper-parameters as used in the previous experiments except  for the number of hidden units. The hidden units for the two multi-layer algorithms are searched in $\{8, 16, 32, 64\}$ instead.

We first tested the decrease in prediction error for the five algorithms with actions randomly selected.  Hyper-parameters with the lowest cumulative square prediction errors were selected. As shown in Figure \ref{fig:email_predict}, our multi-task learning algorithms have much lower cumulative prediction errors. The total prediction errors converge after 2500 steps.

\begin{figure}[h]
    \centering
    \includegraphics[width = 0.85\textwidth]{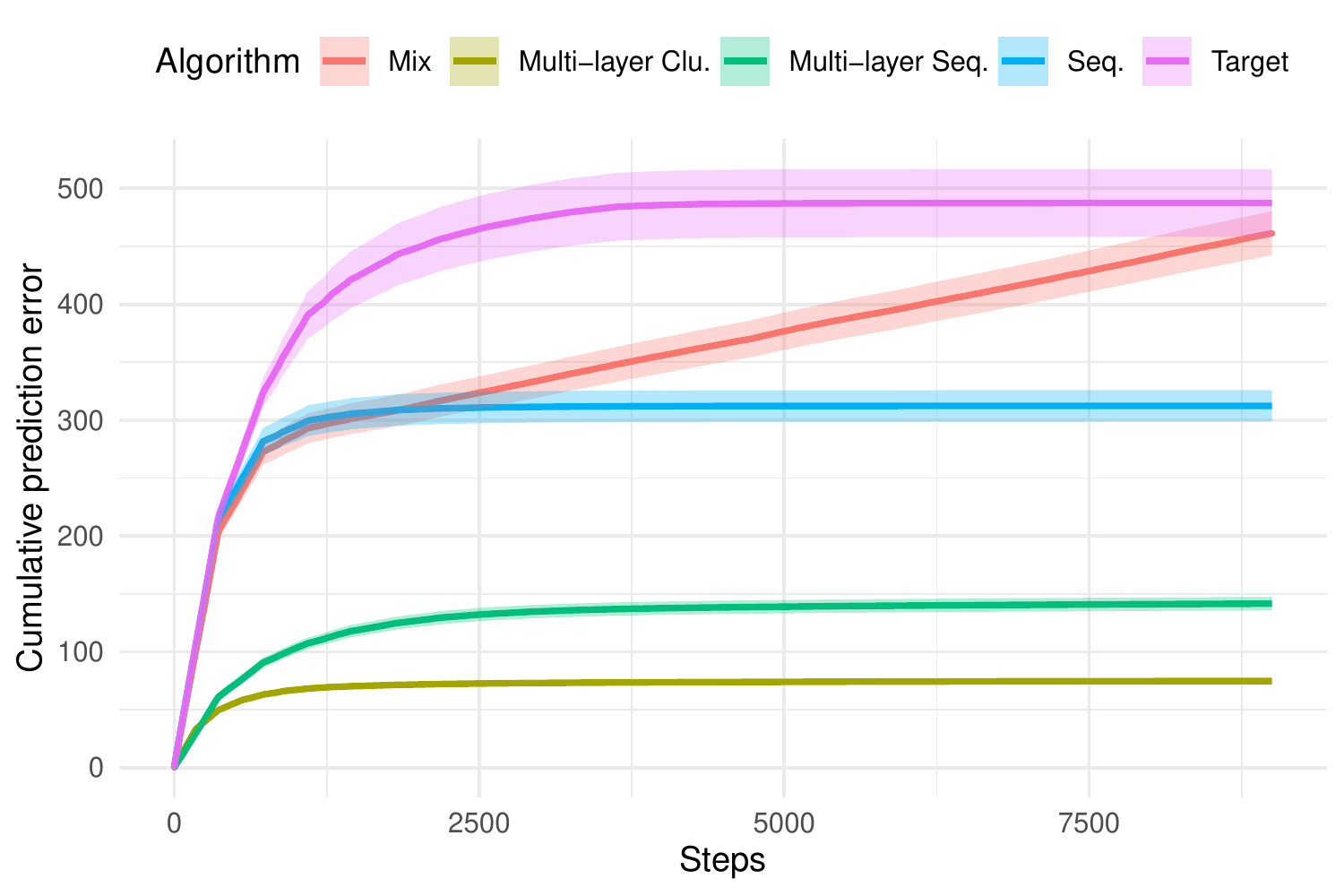}
    \caption{The cumulative square errors for five algorithms. The solid lines are averaged over 10 independent runs and regions mark the 1 standard deviation over the 10 runs.}
    \label{fig:email_predict}
\end{figure}

We further applied our practical algorithm in Algorithm \ref{algo:prac_mt_cb} and selected the hyper-parameters with the lowest regret. Table 1 showed the average increased number of Purchase, Click and Open within 10000 steps with respect to a purely random policy over 10 independent runs. The results are shown in Table \ref{tab:email_results}. As one can see in the table, Multi-layer clustered improved Purchase rate the most by 0.0634 over 10000 steps. However, Mix improved the average number of Click or Open the most. This indicates that the three tasks are not exactly the same.

\section{Discussion}

In this paper, we formulated an important problem, funnel structure, from the marketing field. We used a multi-task learning algorithm to solve the contextual bandit problem with a funnel structure and offered its regret analysis. We verified our theorem using a simple simulation environment and tested the performances of our algorithm on both simulation and real-data environment.

Note that our bounds on prediction error in Theorem \ref{thm:seq_thm} and Theorem \ref{thm:clst_thm} do not scale with $1/\sqrt{\sum_{j} n_j}$ under the special case when $\theta_1 = \dots = \theta_J$. However, the sparsity of the funnel implies that the optimal rate is only smaller than our bound by a constant factor that does not depend on $J$. To see this, assuming that $n_{j+1} / n_{j} = q$, we have $\sum_j n_j \rightarrow n/(1-q)$, which is a constant value.

In terms of the real-data environment, we adopted the population model that may lead to high variance in context and action pairs that do not have sufficient observations and our environment may not fully reflect the true environment. A better data-based environment may be proposed using rare event simulation \cite{rubino2009rare}.

\bibliographystyle{apalike}
\bibliography{main}

\onecolumn

\appendix
\section{Missing Proofs}

\subsection{Connection to discrepancy measure}
\label{app:lem1}

In this section, we discuss how our assumption relates to discrepancy assumptions.
Consider $\mathcal{Y}$-discrepancy that measures the maximum absolute distance between the loss function:
$
    \operatorname{dist}\left(\mathcal{D}_{1}, \mathcal{D}_{2}\right) \coloneqq \sup _{h \in \mathcal{H}}\left|\mathcal{L}_{\mathcal{D}_{1}}(h)-\mathcal{L}_{\mathcal{D}_{2}}(h)\right|, \label{equ:disc}
$
where $\mathcal{D}_{1}$ and $\mathcal{D}_{1}$ represents the source domain and target domain and $\mathcal{L}_{\mathcal{D}_{1}}$ and  $\mathcal{L}_{\mathcal{D}_{2}}$ are expected loss for two domains.

Note that under the GLM assumption, the $L_2$ distance in unknown parameters resembles the discrepancy using square loss. Consider a funnel with two layers and $\|\theta_1 - \theta_2\|_2 = q$. Lemma \ref{lem:disc} indicates that $q \approx \operatorname{dist}(\mathcal{D}_{1}, \mathcal{D}_{2})$.

\begin{lem}
\label{lem:disc}
    We have under square loss function, $ \operatorname{dist}(\mathcal{D}_{1}, \mathcal{D}_{2}) \leq 4\kappa d_x q$.
\end{lem}

{\it Proof.}

We first show the second inequality.
\begin{align*}
    &\quad \operatorname{dist}\left(\mathcal{D}_{1}, \mathcal{D}_{2}\right)\\
    &= \sup_{\theta} |\E_x (\mu(x^T\theta)) -  \mu(x^T\theta_1^*))^2 -
    \E_x(\mu(x^T\theta) - \mu(x^T\theta_2^*))^2|\\
    &\leq \sup_{\theta} |\E_x \mu(x^T\theta)(\mu(x^T\theta_1^*) - \mu(x^T\theta_2^*))| + |\E_x(\mu^2(x^T\theta_1^*) - \mu^2(x^T\theta_2^*))|\\
    &\leq  4|\E_x(\mu(x^T\theta_1^*) - \mu(x^T\theta_2^*))|\\
    &\leq  4\E_x|\mu(x^T\theta_1^*) - \mu(x^T\theta_2^*)|\\
    &\leq  4\kappa\E_x|x^T(\theta_1^* - \theta_2^*)|\\
    &\leq  4\kappa d_x q\\
\end{align*}

On the other hand, an lower bound of $\operatorname{dist}\left(\mathcal{D}_{1}, \mathcal{D}_{2}\right)$ is also closely related to $q$.
\begin{align*}
        &\quad \operatorname{dist}\left(\mathcal{D}_{1}, \mathcal{D}_{2}\right)\\
        &=\sup_{\theta} |\E_x (\mu(x^T\theta)) -  \mu(x^T\theta_1^*))^2 - \E_x(\mu(x^T\theta) - \mu(x^T\theta_2^*))^2|\\
        &= \sup_{\theta} |\E_x(\mu(x^T\theta_1^*) - \mu(x^T\theta_2^*))(\mu(x^T\theta_1^*) + \mu(x^T\theta_2^*) + \mu(x^T \theta))|\\
        &= \sup_{\theta}|\E_x \int_t \mu^{\prime}\left(t x^T\theta_1^* + (1-t)x^T\theta_2^*\right) dt (x^T( \theta_1^* - \theta_2^*)) (\mu(x^T\theta_1^*) + \mu(x^T\theta_2^*) + \mu(x^T \theta))|\\
        &= \sup_{\theta}|(\theta_1^* - \theta_2^*)^T  [\E_x x \int_t \mu^{\prime}\left(t x^T\theta_1^* + (1-t)x^T\theta_2^*\right) dt (\mu(x^T\theta_1^*) + \mu(x^T\theta_2^*) + \mu(x^T \theta))]|\\
        & \quad \text{(Let $\theta \rightarrow -\infty$)}\\
        &\geq |(\theta_1^* - \theta_2^*)^T \nu_{\theta_1^*, \theta_2^*}| \ (\text{letting } \nu_{\theta_1^*, \theta_2^*} = [\E_x x \int_t \mu^{\prime}\left(t x^T\theta_1^* + (1-t)x^T\theta_2^*\right) dt (\mu(x^T\theta_1^*) + \mu(x^T\theta_2^*))]).
\end{align*}

Let $\theta_2^* = \theta_1^* + \|\theta_1^* - \theta_2^*\|_2 \mu$, where $\mu$ is a unit vector. For sufficient small $\|\theta_1^* - \theta_2^*\|_2, \nu_{\theta_1^*, \theta_2^*} \rightarrow 2\E_x[x \mu'(x^T \theta_1^*) \mu(x^T \theta_1^*)] \eqqcolon \nu_{\theta_1^*}$, which is a constant vector. Thus
$$
    \lim_{\|\theta_1^* - \theta_2^*\|_2 \rightarrow 0} \frac{\operatorname{dist}(\mathcal{D}_{1}, \mathcal{D}_{2})}{\|\theta_1^* - \theta_2^*\|_2} = |\mu^T \nu_{\theta_1^*}|.
$$

For sufficient small $\|\theta_1^* - \theta_2^*\|$, discrepancy scales with $\|\theta_1^* - \theta_2^*\|$.

\subsection{Proof of Lemma \ref{lem:pred_err}}
\label{app:lem1}

In this subsection, we introduce the proof of Lemma \ref{lem:pred_err}. Many proofs could achieve a very similar bound. Here we use the idea of local Rademacher complexity.

{\it Proof. }  \\

We discuss two cases: 1) $\hat \theta \in \text{int}(\Theta_0)$. 2) $\hat \theta \notin \text{int}(\Theta_0)$.

In both cases, one simply has
$$
|\mu(x^T \hat \theta) - \mu(x^T \theta^*)| \leq \kappa |x^T(\hat \theta - \theta^*)| \leq \kappa \sup_{\theta_1, \theta_2 \in \Theta_0} |x^T(\theta_1 - \theta_2)|,
$$
which completes the first term in the minimum.

Now we prove the parametric bound. We first assume that case 1 holds. In this case, the constraint does not come into effects and $\hat\theta$ is the global minimal.
By Theorem 26.5 in \cite{shalev2014understanding}, we have under an event, whose probability is at least $1-\delta$,
\begin{equation}
\label{equ:gen_error}
    L(\hat \theta) - L(\theta^*) \leq 2 R_n(\vect{z}) + 5 \sqrt{\frac{2 \ln (8 / \delta)}{n}},
\end{equation}
where $R(\vect{z})$ is the Rademacher complexity defined by
$$
    R_n(\vect{z}) = \E_{\vect{\sigma}} \frac{1}{n} \sup_{\theta \in \Theta} \sum_{i = 1}^n \|z_i - \mu(x_i^T\theta)\|_{M_n}^2 \sigma_i,
$$
and the variables in $\vect{\sigma}$ are distributed i.i.d. from Rademacher distribution. Let us call the event $E_A$.

As for any $i \in [n]$, let $\phi_i(t) \coloneqq  (z_i - \mu(t))^2$, which satisfies $|\phi_i^{\prime}(t)| = |2(z_i - \mu(t))\mu'(t)| \leq \kappa$, using Contraction lemma \citep{shalev2014understanding}, we have
\begin{align*}
    R_n(\vect{z})
    &\leq \E_{\vect{\sigma}} \frac{1}{n} \sup_{\theta \in \Theta} \sum_i \kappa (x_i^T\theta)\sigma_i\\
    &= \kappa\E_{\vect{\sigma}} \frac{1}{n} \sup_{\theta \in \Theta} \sum_{i = 1}^n x_i^T (\theta - \theta^*) \sigma_i. \numberthis \label{equ:rad1}\\
    &\leq \kappa\E_{\vect{\sigma}} \frac{1}{n} \sup_{\theta \in \Theta} \|\sum_ix_i\sigma_i\|_{M_n^{-1}} \|\theta - \theta^*\|_{M_n} \\
    &\leq \kappa\E_{\vect{\sigma}} \frac{1}{n}  \|\sum_ix_i\sigma_i\|_{M_n^{-1}} \sup_{\theta \in \Theta_0} \|\theta - \theta^*\|_{M_n}.
\end{align*}

Next, using Jensen's inequality we have that
\begin{align*}
    &\quad \E_{\vect{\sigma}} \frac{1}{n}  \|\sum_ix_i\sigma_i\|_{M_n^{-1}} \\
    &\leq \frac{1}{n}  \left( \E_{\vect{\sigma}}  \|\sum_ix_i\sigma_i\|^2_{M_n^{-1}} \right)^{1/2} \\
    &= \frac{1}{n}  \left( \E_{\vect{\sigma}} tr[M_n^{-1} (\sum_ix_i\sigma_i) (\sum_ix_i\sigma_i)^T]  \right)^{1/2} \\
    &= \frac{1}{n}  \left( tr[M_n^{-1} \E_{\vect{\sigma}}(\sum_ix_i\sigma_i) (\sum_ix_i\sigma_i)^T]  \right)^{1/2} \numberthis \label{equ:trace1}
\end{align*}

Finally, since the variables $\sigma_{1}, \ldots, \sigma_{m}$ are independent we have
\begin{align*}
    &\quad \E_{\vect{\sigma}}(\sum_ix_i\sigma_i) (\sum_ix_i\sigma_i)^T\\
    &= \E_{\vect{\sigma}} \sum_{k, l\in [n]} \sigma_k\sigma_l x_k x_l^T \\
    &= \E_{\vect{\sigma}} \sum_{i \in [n]} \sigma_i^2 x_i x_i^T \\
    &= \sum_{i \in [n]} x_i x_i^T = nM_n.
\end{align*}

Plugging this into (\ref{equ:trace1}), assuming $M_n$ is full rank, we have \begin{equation}
    (\ref{equ:rad1}) \leq \sqrt{d/n}\sup _{\theta \in \Theta_0}\|{\theta}-\theta^{*}\|_{M_{n}}.
\end{equation}

\begin{lem}
\label{lem:2}
Under the notation in Lemma \ref{lem:pred_err} and Assumption \ref{assp:cts_convx}, if an estimate $\hat \theta$ satisfies $L(\hat \theta) \leq L(\theta^*) + b_n$, then
$$
    \|\hat \theta\ - \theta^*\|_{M_n}^2 \leq \frac{d_x b_n}{c_{\mu}}.
$$
\end{lem}

\begin{proof}
Let $g_n(\theta) = \sum_{i} x_i (\mu(x_i^T \theta) - \mu(x_i^T \theta^*))$. For any $\theta$, $\nabla g_n(\theta) = \sum_i x_ix_i^T \mu'(x_i^T \theta)$. By simple calculus,
$$
    g_n(\theta^*) - g_n(\hat\theta) = \int_{0}^{1} \nabla g_{n}\left(s \theta^{*}+(1-s) \hat{\theta}\right) d s (\theta^* - \hat \theta).
$$
As $\mu(t) \geq c_{\mu}$, we have $\int_{0}^{1} \nabla g_{n}\left(s \theta^{*}+(1-s) \hat{\theta}\right) d s \succ c_{\mu} M_n$. Plugging this into the inequality above we have
$$
    \|\theta^* - \hat \theta\|^2_{M_n} \leq \frac{1}{c_{\mu} } (\sum_{i} x_i (\mu(x_i^T \theta) - \mu(x_i^T \theta^*)))^2 = \frac{1}{c_{\mu} } \epsilon^T M_n \epsilon \leq \frac{d_x}{c_{\mu} } \epsilon^T \epsilon = \frac{d_x}{c_{\mu}} (L(\hat\theta) - L(\theta^*)),
$$
where $\epsilon \coloneqq (\mu(x_i^T \hat \theta) - \mu(x_i^T \theta^*))_{i = 1}^n$.

Applying (\ref{equ:gen_error}) and Lemma \ref{lem:2}, we complete the proof by
\begin{align*}
    \|\hat \theta - \theta^*\|_{M_n}
    &\leq \sqrt{\frac{2d_x\sqrt{d}}{c_{\mu}\sqrt{n}} \sup_{\theta \in \Theta} \|\theta - \theta^*\|_{M_n} + 5 \sqrt{\frac{2 \ln (8 / \delta)}{n}}}\\
    &\leq \sqrt{\frac{20\sqrt{2\ln(8/\delta)}d_x\sqrt{d}\sup_{\theta \in \Theta_0} \|\theta - \theta^*\|_{M_n}}{c_{\mu}\sqrt{n}}}. \numberthis \label{equ:iter}
\end{align*}
\end{proof}

We apply (\ref{equ:iter}) iteratively \footnote{Note that (\ref{equ:iter}) holds under the same event $E_A$ as the estimates $\hat \theta$ keeps the same each round as it is the global minimizer.}. Let $\Theta_{(1)} \coloneqq \Theta_0$. For any $t > 1$, let $\Theta_{(t)} = \{\theta \in \mathbb{R}^d: \|\theta - \hat \theta\|_{M_n} \leq \sqrt{\frac{20\sqrt{2 d \ln (8 / \delta)}}{c_{\mu}\sqrt{n}} \sup_{\theta \in \Theta_{(t-1)}} \|\theta - \theta^*\|_{M_n}}\}$. When $t \rightarrow \infty$, we have
$$
    \Theta_{(\infty)} = \frac{20d_x\sqrt{2 d \ln (8 / \delta)}}{c_{\mu}\sqrt{n}}.
$$

By (\ref{equ:iter}), we have $\theta^* \in \cap_{t \geq 1} \Theta_{(\infty)}$ and $\|\hat \theta - \theta^*\|_{M_n} \leq \frac{40d_x\sqrt{2 d \ln (8 / \delta)}}{c_{\mu}\sqrt{n}}$, which completes the second part of Lemma \ref{lem:pred_err}.

For any $x \in \mathcal{X}$, we have
\begin{equation}
    |\mu(x^T \hat \theta) - \mu(x^T\theta^*)| \leq \kappa \|x\|_{M_n^{-1}} \frac{40d_x\sqrt{2 d \ln (8 / \delta)}}{c_{\mu}\sqrt{n}}.
\end{equation}

When case 2 holds, let  $\hat\theta'$ be the global minimizer. Using the analysis above, we have
$$
\|\hat{\theta}'-\theta^{*}\|_{M_{n}} \leq \frac{40 d_{x} \sqrt{2 d \ln (8 / \delta)}}{c_{\mu} \sqrt{n}}.
$$
Then by triangle inequality
$$
    \|\hat{\theta}-\theta^{*}\|_{M_{n}} \leq \|\hat{\theta}-\hat\theta'\|_{M_{n}} + \|\hat{\theta}'-\theta^*\|_{M_{n}} \leq \frac{80 d_{x} \sqrt{2 d \ln (8 / \delta)}}{c_{\mu} \sqrt{n}}.
$$

\subsection{Tightness of Lemma \ref{lem:pred_err}}
\label{app:tightness}
We use an example to show the tightness of Lemma \ref{lem:pred_err}. Assume a linear predictor, i.e. $\mu(t) = t$. Consider the following distribution, let $X$ be uniform over the $d$-standard basis vector $e_m$, for $m = 1, \dots, d$. Let $Z \mid (X = e_i) \sim Bern(r_i)$, where $r_i \in [0, 1]$ is pre-determined and unknown. The optimal parameter $\theta^* = (r_1, \dots, r_d)^T$. Let $n_m$ be the number of samples collected for dimension $m$. Let $\Theta_0 \coloneqq \{\theta: \|\theta\|_2\leq q\}$.

When $n$ is sufficiently large $n > 1/q^2$, $\hat\theta$ is the regularized minimizer.  It can be shown that for any $\hat \theta$, there exists $\theta^*$ such that $\E[\hat \theta_i - \theta^*_i]^2 \geq (r_m(1-r_m)) / n_m$.
Then $\E\|\hat \theta - \theta^*\|^2_2 \geq \sum_{m = 1}^d \frac{r_i(1-r_i)}{n_m} \geq \frac{d^2(r_i(1-r_i))}{n} = \Omega(\frac{d^2}{n})$.

Then we also see that when $n$ is small ($\leq \frac{1}{q^2}$), the estimation error is $\Omega(q)$. We use the same example as above. This time, we assume $\|\theta^*\| \leq \frac{q}{2}$. If we have a $\|\hat \theta\| = q$, then $\|\theta^* - \hat \theta\| \geq q/2 = \Omega(q)$. Otherwise, we use the lower bound above: $\|\theta^* - \hat \theta\| \geq \Omega(\frac{d}{\sqrt{n}}) = \Omega(dq)$.

The above argument corresponds to the upper bound in Lemma \ref{lem:pred_err}, where we use prior knowledge when $n$ is small and use the parametric bound when $n$ is large.

\subsection{Proof of Theorem \ref{thm:seq_thm}}

In this subsection, we show the missing proof for Theorem \ref{thm:seq_thm}.
\begin{thm}[Prediction error under sequential dependency]
For any funnel with a sequential dependency of parameters $q_1, \dots, q_J$, let $\hat \theta_1, \dots, \hat \theta_J$ be the estimates from Algorithm \ref{algo:mt_funnel}. If $n_{j+1} \leq n_j / 4$, $q_1 \geq, \dots, \geq q_J $ and Assumption 5 is satisfied, then with a probability at least $1-\delta$, for any $j_0 \in [J]$, we have
\begin{equation}
    \text{PE}_j \leq
    \left\{
        \begin{array}{ll}
             & \kappa \|x\|_2 \frac{ c_{\delta}}{c_{\mu}\lambda} \sqrt{\frac{d}{n_j}}, \text{ if } j < j_0, \\
             & \kappa\|x\|_2(\frac{ c_{\delta}}{c_{\mu}\lambda} \sqrt{\frac{d}{n_{j_0}}} + \sum_{i = j_0 + 1}^{j} q_j), \text{ if } j \geq j_0,
        \end{array} \label{equ:thm_pred_1}
    \right.
\end{equation}
where we let $n_0 = \infty$.
The bound is smallest when $j_0$ is the smallest $j \in [J]$, such that
\begin{equation}
    \frac{4 c_{\delta}\sqrt{d}}{c_{\mu} \lambda} (\frac{1}{\sqrt{n_{j}}} - \frac{1}{\sqrt{n_{j-1}}}) \geq q_{j}, \label{equ:condition1}
\end{equation}
if none of $j$'s in $[J]$ satisfies (\ref{equ:condition1}), $j_0 = J+1$.
\end{thm}

{\it Proof.}
First we reshape the ellipsoid in (\ref{equ:lemma1_2}) to a ball.
\begin{lem}[Reshape]
\label{lem:reshape}
For any vector $x \in \mathbb{R}^d$ and any matrix $M \succ 0 \in \mathbb{R}^{d \times d}$,
$
 \|x\|_2 \leq \frac{1}{\lambda}\|x\|_{M},
$
where $\lambda$ is the minimum eigenvalue of $M$.
\end{lem}

\begin{proof}
We directly use the definition of positive definite matrix:
$
\lambda^2\|x\|^2_2 - \|x\|^2_{M} = x^T (\lambda^2 I - M) x \leq 0.
$
Thus, $\|x\|_2 \leq \frac{1}{\lambda^2} \|x\|_{M}$. \#

Using Lemma \ref{lem:reshape} and Assumption \ref{asp:covairate_dist}, we have
$
    \| \bar \theta_j - \theta^*_j\|_2 \leq \frac{1}{\lambda} \|\bar \theta_j - \theta_j^*\|_{M_n} \leq
    \frac{4c_{\delta}}{c_{\mu}\lambda}\sqrt{\frac{d}{n}}.
$ Thus the set $\hat \Theta_j \subset \{\theta: \|\theta - \bar\theta_j\|_2 \leq \frac{4c_{\delta}}{c_{\mu}\lambda}\sqrt{\frac{d}{n}} \} \eqqcolon \hat \Theta^{ball}_j$. 

For every $j$, one can derive two bounds. First we can directly apply Corollary \ref{cor:1} and get $PE_j \leq \kappa\|x\|_2\frac{4 c_{\delta}}{c_{\mu} \lambda} \sqrt{\frac{d}{n_j}}$. Second, for any $j_0$, we have $\theta_j^* \in \Theta_1[j] \subset \{\theta: \|\bar \theta_{j_0} - \theta\|_2 \leq \frac{4 c_{s}}{c_{\mu} \lambda} \sqrt{\frac{d}{n_{j_0}}} + \sum_{j_0+1 \leq i \leq j} q_i\}$ and get $PE_j \leq \kappa\|x\|_2(\frac{ c_{\delta}}{c_{\mu}\lambda} \sqrt{\frac{d}{n_{j_0}}} + \sum_{i = j_0 + 1}^{j} q_j)$.

Now we show the second argument: of all those bounds the one defined in (\ref{equ:thm_pred_1}) with $j_0$ defined in (\ref{equ:condition1}) is the smallest. For any $j \leq j_0$ and $j_1 \leq j$, we have
\begin{equation}
    \frac{4 c_{\delta}}{c_{\mu} \lambda} \sqrt{\frac{d}{n_{j}}} = \frac{4 c_{\delta}\sqrt{d}}{c_{\mu} \lambda} \left(\sum_{i = j_1 + 1}^{j} (\frac{1}{\sqrt{n_i}} - \frac{1}{\sqrt{n_{i-1}}}) + \frac{1}{\sqrt{n_{j_1}}} \right) \leq \frac{4 c_{\delta}}{c_{\mu} \lambda} \sqrt{\frac{d}{n_{j_1}}} + \sum_{i = j_1+1}^{j} q_i. \label{equ:thm1_step1}
\end{equation}

The second inequality is given by
$(\frac{1}{\sqrt{n_i}} - \frac{1}{\sqrt{n_{i-1}}}) \leq q_i$ for all $i < j_0$. For any $j \geq j_0$ and $j_1 \leq j_0$, by (\ref{equ:thm1_step1}), we have
$$
    \frac{4 c_{\delta}}{c_{\mu} \lambda} \sqrt{\frac{d}{n_{j_{0}}}} + \sum_{i = j_0+1}^{j} q_i \leq \frac{4 c_{\delta}}{c_{\mu} \lambda} \sqrt{\frac{d}{n_{j_{1}}}} + \sum_{i = j_1+1}^{j} q_i.
$$

Now we prove that for all $i \geq j_0$,
\begin{equation}
\label{equ:condition2}
    \frac{4 c_{\delta} \sqrt{d}}{c_{\mu} \lambda}(\frac{1}{\sqrt{n_i}} - \frac{1}{\sqrt{n_{i-1}}}) \geq q_i.
\end{equation}
We use induction. Assume for some $i_1$, (\ref{equ:condition2}) is satisfied. Under the assumption that $n_{i_1-1} \leq n_{i_1} / 4$ and $q_{i_1} \geq q_{i_1 + 1}$, we have
\begin{align*}
    \frac{4 c_{\delta} \sqrt{d}}{c_{\mu} \lambda}(\frac{1}{\sqrt{n_{i_1 + 1}}} - \frac{1}{\sqrt{n_{i_1}}})
    &= \frac{4 c_{\delta} \sqrt{d}}{c_{\mu} \lambda}(\frac{1}{\sqrt{n_{i_1 + 1}}} + \frac{1}{\sqrt{n_{i_1 - 1}}} - \frac{2}{\sqrt{n_{i_1}}} + \frac{1}{\sqrt{n_{i_1}}} - \frac{1}{\sqrt{n_{i_1-1}}})\\
    &\geq \frac{4 c_{\delta} \sqrt{d}}{c_{\mu} \lambda}(\frac{1}{\sqrt{n_{i_1 + 1}}} + \frac{2}{\sqrt{n_{i_1}}} - \frac{2}{\sqrt{n_{i_1}}} + \frac{1}{\sqrt{n_{i_1}}} - \frac{1}{\sqrt{n_{i_1-1}}})\\
    &\geq \frac{4 c_{\delta} \sqrt{d}}{c_{\mu} \lambda}( + \frac{1}{\sqrt{n_{i_1}}} - \frac{1}{\sqrt{n_{i_1-1}}})\\
    &\geq q_{i_1} \geq q_{i_1 + 1}.
\end{align*}
Using \ref{equ:condition2}, for any $j \geq j_1 > j_0$,
$$
    \frac{4 c_{\delta}}{c_{\mu} \lambda} \sqrt{\frac{d}{n_{j_{0}}}} + \sum_{i = j_0+1}^{j} q_i = \frac{4 c_{\delta}\sqrt{d}}{c_{\mu} \lambda} \left(\sum_{i = j_0 + 1}^{j_1} (\frac{1}{\sqrt{n_{i-1}}} - \frac{1}{\sqrt{n_{i}}}) + \frac{1}{\sqrt{n_{j_1}}} \right) + \sum_{i = j_0+1}^{j} q_i \leq \frac{4 c_{\delta}}{c_{\mu} \lambda} \sqrt{\frac{d}{n_{j_{1}}}} + \sum_{i = j_1+1}^{j} q_i.
$$

Finally, we conclude that $j_0$ gives the smallest bound. \#
\end{proof}

Similar argument can be used to show Theorem \ref{thm:clst_thm}.

\subsection{Proof of Theorem \ref{thm:regret}}
\label{app:thm_regret}

\begin{thm}
Using Algorithm \ref{algo:mt_cb}, under the Assumptions 1-4, with a probability at least $1-\delta$, the total regret
\begin{align*}
    &\quad \sum_{t=1}^{T}\left[P(x_t, \vect{\theta}_{a_t^*}^*)- P(x_t, \vect{\theta}_{a_t}^*)\right] \\
    &\leq 2\sqrt{2}c_{0} \sum_{a, j} \sqrt{n_{a, j}^{T}} + \sum_{a, j} \frac{8c_0^2Jd_x^4 \log (6AJT / \delta)}{ \bar p_{a, j}^2}  - \sum_{a, j} \Delta_{a, j}.
    \numberthis \label{equ:regret}
\end{align*}
where $\mathcal{O}$ ignores all the constant terms and logarithmic terms for better demonstrations, $c_0 = ({\kappa d_x c_{\delta/AJT} \sqrt{d}})/({c_{\mu} \bar{\lambda}})$, $\bar p_a \coloneqq \E_x P_{J-1}(x^T \vect{\theta}_{a}^{*})$ and
$$
    \Delta_{a, j} = \sum_{t = 1; a_t = a}^T P_j(x_t^T \hat\theta_{a_t}^t) \left[c_0 \frac{1}{\sqrt{n^t_{a, j}\vee 1}} - \Delta \mu_{a, j}^t \right].
$$
represents the benefits of transfer learning.
\end{thm}

Let $\bar p_{a, j} \coloneqq \mathbb{E}_{x} P_{j-1}\left(x^{T} \theta_{a}^{*}\right)$. We first show that upper bound the number of steps $t$ with $\lambda_{a_t, j}^t \leq \bar \lambda / 2 \text{ or } n_{a, j}^{t} \leq \frac{1}{2} n_{a, 1}^{t} \bar p_{a,j}$. These steps are considered bad events.

Lemma \ref{lem:low_obs_bound} shows that with high probability, the number of observations for each layer is close to its expectation.

\begin{lem}
\label{lem:num_obs}
\label{lem:low_obs_bound}
With a probability at least $1-\delta$, we have
$n_{a, j}^t \geq n_{a, 1}^t \bar p_{a, j} - \sqrt{2n_{a, 1}^t\log(1/\delta)}.$ Especially, when $n_{a, 1}^t > 8\log(1/\delta) / \bar p_{a, j}^2 \eqqcolon c_{n, a}$, we have  $n_{a, j}^t \geq \frac{1}{2}n_{a, 1}^t \bar p_{a, j}$.
\end{lem}
\begin{proof}
This is a direct application of Hoeffding inequality.
\end{proof}

\begin{lem}
\label{lem:lambda_concen}
For any $x_1, \dots, x_n$ i.i.d, $\|x_i\|\leq d_x$, let $\lambda_n$ be the minimum eigenvalue of $\sum_i x_ix_i^T / n$ and $\bar\lambda$ be the minimum eigenvalue of its expectation. We have
$
   \lambda_n \geq \bar\lambda / 2,
$
when $n > d_{x}^{4} \log (1 / \delta) / \bar{\lambda}^{2}$.
\end{lem}
\begin{proof}
For all $x_1, \dots, x_n$, write $x_{i}=\sum_{s=1}^{d} \nu_{s, i} \tilde{x}_{s}$, where $\tilde{x}_{1}, \dots, \tilde{x}_{d}$ are any basis of $\mathbb{R}^d$. We have $\E \nu_{s, i}^2 \geq \bar \lambda$. For Hoeffding's inequality, since $\nu_{s, i} \leq d_x$, with a probability $1-\delta$, we have
$$
    \frac{1}{n} \sum_{i} \nu_{s, i}^{2} \geq \mathbb{E} \nu_{s, 1}^{2}-d_{x}^{2} \sqrt{\frac{\log (1 / \delta)}{n}} \geq \bar \lambda - d_{x}^{2} \sqrt{\frac{\log (1 / \delta)}{n}}.
$$
For $n > d_x^4\log(1/\delta) / \bar\lambda^2$, we have $\frac{1}{n} \sum_{i} \nu_{s, i}^{2} \geq \bar \lambda/ 2$. There exists a choice of $\tilde{x}_{1}, \dots, \tilde{x}_{d}$ such that $\lambda_n = \frac{1}{n} \sum_{i} \nu_{s, i}^{2}$.
\end{proof}

Combining Lemma \ref{lem:num_obs} and Lemma \ref{lem:lambda_concen}, we have with a probability at least $1-\delta/3$, $\#\{t:\exists j,  \lambda_{a_t, j}^t \leq \bar \lambda / 2 \text{ or } n_{a, j}^{t} \leq \frac{1}{2} n_{a, 1}^{t} \bar p_{a,j}\}$ can be upper bounded by
\begin{equation}
    \sum_{a, j} \max\left\{8 \log (6AJT / \delta) /\bar p_{a, j}^{2}, 2d_{x}^{4} \log (6AJT / \delta) / (\bar{\lambda}^{2}\bar p_{a, j})\right\}. \label{equ:equ0}
\end{equation}

In the following proof, we assume for all $t$, $\lambda_{a,j}^t \geq \bar\lambda / 2$ and $n_{a, j}^t \geq \frac{1}{2} n_{a, 1}^{t} \bar{p}_{a, j}$. We also assume the event in Lemma \ref{lem:pred_err} happens for all $a\in [A], j \in [J]$ and $t < T$. The probability is at least $1-\delta/3$ as each probability is at least $1 - \delta/(3AJT)$.

The total regret is
\begin{align*}
    &\quad \sum_{t=1}^{T}\left[P(x_{t}, \vect{\theta}_{a_{t}^{*}}^{*})-P(x_{t}, \vect{\theta}_{a_{t}}^{*})\right] \\
    &\leq \sum_{t=1}^{T} \left[P(x_{t}, \vect{\theta}_{a_{t}^{*}}^{*}) - P^+(x_{t}, \hat{\vect{\theta}}_{a_{t}}) + P^+(x_{t}, \hat{\vect{\theta}}_{a_{t}}) - P(x_{t}, \vect{\theta}_{a_{t}}^{*}) \right] \\
    & (\text{Using }  P(x_{t}, \vect{\theta}_{a_{t}^{*}}^{*}) - P^+(x_{t}, \hat{\vect{\theta}}_{a_{t}}) \leq 0) \\
    &\leq  \sum_{t=1}^{T} \left[P^+(x_{t}, \hat{\vect{\theta}}^t_{a_{t}}) - P(x_{t}, \vect{\theta}_{a_{t}}^{*}) \right] \\
    & \text{(Using Lemma \ref{lem:expansion})}\\
    &\leq \sum_{t=1}^{T}\left[ \sum_{j} \frac{P_{J}\left(x, \hat{\theta}_{a_t}^{t}\right)}{\mu\left(x^{T} \hat{\theta}_{a_t, j}^{t}\right)} \Delta \mu_{a_t, j}^{t}+\sum_{i \neq j} \Delta \mu_{a_t, j}^{t} \Delta \mu_{a_t, i}^{t}\right]\\
    &\leq \sum_{t=1}^{T}\left[ \sum_{j} P_{j}(x_t, \hat \theta^t_{a_t}) \Delta \mu_{a_t, j}^{t}+\sum_{i \neq j} \Delta \mu_{a_t, j}^{t} \Delta \mu_{a_t, i}^{t}\right]\\
    &= \sum_{t=1}^{T}\left[ \sum_{j} (P_{j}(x_t, \theta^*_{a_t}) + P_{j}(x_t, \hat \theta^t_{a_t}) - P_{j}(x_t, \theta^*_{a_t})) \Delta \mu_{a_t, j}^{t}+\sum_{i \neq j} \Delta \mu_{a_t, j}^{t} \Delta \mu_{a_t, i}^{t}\right]\\
    &\leq
    \underbrace{\sum_{t = 1}^T \sum_j P_{j}\left(x_{t}, \theta_{a_{t}}^{*}\right) \frac{c_0}{\sqrt{n_{a_t,j}^t}}}_{\textcircled{1}}
    - \underbrace{\sum_{t = 1}^T \sum_j P_{j}\left(x_{t}, \theta_{a_{t}}^{*}\right) (\frac{c_0}{\sqrt{n_{a_t,j}^t}} - \Delta \mu_{a_t, i}^{t})}_{\textcircled{2}}
    + \underbrace{\sum_{t=1}^T \sum_{i \neq j} \Delta \mu_{a_t, j}^{t} \Delta \mu_{a_t, i}^{t}}_{\textcircled{3}} \\
    &\quad + \underbrace{\sum_{t=1}^{T}\left[ \sum_{j} (P_{j}(x_t, \hat \theta^t_{a_t}) - P_{j}(x_t, \theta^*_{a_t})) \Delta \mu_{a_t, j}^{t} \right]}_{\textcircled{4}}.
\end{align*}

We further bound the terms separately. The first term $\textcircled{1}$ represents the bound one could have without multi-task learning.

\begin{align*}
    &\quad \sum_{t = 1}^T \sum_j P_{j}(x_t, \theta^*_{a_t}) \frac{c_0}{\sqrt{n_{a_t,j}^t}}\\
    &\leq \sum_{t = 1}^T \sum_j \mathbbm{1}(r_{t, j - 1} = 1)\frac{c_0}{\sqrt{n_{a_t,j}^t}} + \sum_{t = 1}^T \sum_j(P_{j}(x_t, \theta^*_{a_t}) - \mathbbm{1}(r_{t, j - 1} = 1)) \frac{c_0}{\sqrt{n_{a_t,j}^t}} \\
    &\quad \text{(Using Lemma 19 in \cite{jaksch2010near})} \\
    &\leq c_0 2\sqrt{2} \sum_{a, j} \sqrt{n_{a, j}^T} + \sum_{t = 1}^T \sum_j(P_{j}(x_t, \theta^*_{a_t}) - \mathbbm{1}(r_{t, j - 1} = 1)) \frac{c_0}{\sqrt{n_{a_t,j}^t}} \numberthis \label{equ:18}
\end{align*}

As $\E[P_{j}(x_t, \theta^*_{a_t}) - \mathbbm{1}(r_{t, j - 1} = 1)] = 0$, the second term in (\ref{equ:18}) is a martingale. Using Azuma-Hoeffding inequality, with a probability at least $1-\delta/3$, for all $T$,
\begin{equation}
    \sum_{t = 1}^T \sum_{j} (P_{j}(x_t, \theta^*_{a_t}) - \mathbbm{1}(r_{t, j - 1} = 1)) \frac{c_0}{\sqrt{n_{a_t,j}^t}} \leq c_0\sqrt{2\log(3TJ/\delta)}. \label{equ:equ1.5}
\end{equation}

Combined with (\ref{equ:18}),

\begin{equation}
\textcircled{1} \leq 2\sqrt{2} c_{0} \sum_{a, j} \sqrt{n_{a, j}^{T}} + c_{0} \sqrt{2 \log (3T J / \delta)}. \numberthis \label{equ:equ1}
\end{equation}

Next we bound $\textcircled{3}$. We notice that this is a quadratic term. We first show Lemma \ref{lem:low_obs_bound} that lower bounds the number of observations for each layer. Lemma \ref{lem:low_obs_bound} is a direct application of Hoeffding's inequality.

For any pair $i, j$, we have
\begin{align*}
    &\quad \sum_{t=1}^{T}  \Delta \mu_{a_{t}, j}^{t} \Delta \mu_{a_{t}, i}^{t}\\
    &\leq c_0^2 \sum_{t=1}^{T}  \frac{1}{\sqrt{n_{a_t, i}^t}} \frac{1}{\sqrt{n_{a_t, j}^t}}\\
    &\leq c_0^2 \sum_{t=1}^{T} \left[\mathbbm{1}(n^t_{a_t, 1} \leq c_{n, a_t}) \frac{1}{\sqrt{n_{a_t, i}^t}} \frac{1}{\sqrt{n_{a_t, j}^t}} + \mathbbm{1}(n^t_{a_t, 1} > c_{n, a_t}) \frac{1}{\sqrt{n_{a_t, i}^t}} \frac{1}{\sqrt{n_{a_t, j}^t}}\right] \\
    &\leq c_0^2 \sum_a c_{n, a} + c_0^2 \sum_t \frac{4}{\bar p_a^2 n^t_{a_t, 1}} \\
    &\leq c_0^2 \sum_a c_{n, a} + c_0^2 \sum_a \frac{4\log(n_{a, 1}^T)}{\bar p_a^2}\\
    &\leq 4c_0^2\sum_a \frac{\log(n_{a, 1}^TA/\delta)}{\bar p_a^2} \numberthis \label{equ:equ3}.
\end{align*}
where we let $\bar p_a \coloneqq \mathbb{E}_{x} P_{J}\left(x^{T} \theta_{a}^{*}\right)$.

Thus, $\textcircled{3}$ is upper bounded by $4c_0^2 J^2 \sum_a \frac{\log(n_{a, 1}^TA/(3\delta))}{\bar p_a^2}$.

Finally we bound term $\textcircled{4}$. Using Lemma \ref{lem:expansion} on only first $j$ layers, we have
\begin{align*}
    \textcircled{4} \leq \sum_t \sum_j [\sum_i \Delta \mu^t_{a_t, i} + \sum_{i, k} \Delta \mu_{a_{t}, k}^{t} \Delta \mu_{a_{t}, i}^{t}] \Delta \mu_{a_{t}, j}^{t} \leq (J+1) \times \textcircled{3}. \numberthis \label{equ:equ4}
\end{align*}

The proof is completed by combining Equations (\ref{equ:equ0}),
(\ref{equ:equ1.5}),
(\ref{equ:equ1}), (\ref{equ:equ3}) and (\ref{equ:equ4}).

\section{Experiments}

\subsection{Practical algorithm}
\label{app:prac_mt_cb}

\begin{algorithm}[H]
  \caption{Practical Algorithm for Contextual Bandit with a Funnel Structure}\label{algo:prac_mt_cb}
  \begin{algorithmic}
    \STATE $t \rightarrow 1$, total number of steps $T$, memory $\mathcal{H}_{a} = \{\}$ for all $a \in [A]$. Initialize $\hat\theta_{a, \star}$ with zero vectors.
    \STATE $\hat\theta_{a, 0} \rightarrow 0$.
    \FOR{$t = 1$ to $T$}
        \STATE Receive context $x_t$.
        \STATE Choose $a_t = \argmax_{a \in \mathcal{A}} \hat P_{J}(x_t, \hat\theta_{a, j})$.
        \STATE Set $a_t = \text{Unif}([A])$ with probability $\epsilon$.
        \STATE Receive $r_{t,1}, \dots, r_{t,J}$ from funnel $F_{a_t}$.
        \STATE Set $\mathcal{H}_{a_t} \rightarrow \mathcal{H}_{a_t} \cup \{(x_t, (r_{t,1}, \dots, r_{t,J}))\}$.
        \FOR{$j = 1, \dots, J$}
        \STATE {\it \# For sequential dependency}
            $$\hat \theta_{a_t, j} \rightarrow \argmin_{\theta} l(\theta, \mathcal{H}_{a_t}) + \lambda_j \|\theta - \hat \theta_{a_t, j-1}\|_2$$  
        \STATE {\it \# For clustered dependency}
            $$\hat \theta_{a_t, j} \rightarrow \argmin_{\theta} l(\theta, \mathcal{H}_{a_t}) + \lambda_j \|\theta - \frac{1}{J}\sum_{i}\hat \theta_{a_t, i}\|_2$$
        \ENDFOR

    \ENDFOR
  \end{algorithmic}
\end{algorithm}

\subsection{Tuned hyper-parameters}
\paragraph{Simulated environment.}
\begin{enumerate}
    \item Target: units 16
    \item Mix: units 32
    \item Sequential: units 32
    \item Multi-layer Clustered: units 4; $\lambda$ 0.001
    \item Multi-layer Sequential: units 8; $\lambda$ 0.001
\end{enumerate}
\paragraph{Data-based environment.}
\begin{enumerate}
    \item Target: units 64
    \item Mix: units 64
    \item Sequential: units 64
    \item Multi-layer Clustered: units 64; $\lambda$ 0.005
    \item Multi-layer Sequential: units 16; $\lambda$ 0.001
\end{enumerate}
\end{document}